%% file: main.tex
\pdfoutput=1
\documentclass{article}

\usepackage[preprint,nonatbib]{neurips_2026}

\input{paper_preamble}

\renewcommand{\acksection}{\section{Acknowledgments}}

\author{%
  Ali Backour \\
  Massachusetts Institute of Technology \\
  Cambridge, MA 02139 \\
  \texttt{abackour@mit.edu} \\
}

\begin{document}
\input{paper_body}

\end{document}

%% file: paper_preamble.tex
\usepackage[utf8]{inputenc}
\usepackage[T1]{fontenc}
\usepackage{hyperref}
\usepackage{url}
\usepackage{booktabs}
\usepackage{amsfonts}
\usepackage{amsmath}
\usepackage{amsthm}
\usepackage{nicefrac}
\usepackage{microtype}
\usepackage{graphicx}
\usepackage{xcolor}
\usepackage{caption}
\newsavebox{\tabcalib}
\newsavebox{\tabladder}

\hypersetup{colorlinks=true, linkcolor=black, citecolor=black, urlcolor=black}

\theoremstyle{plain}
\newtheorem{theorem}{Theorem}
\theoremstyle{remark}
\newtheorem{remark}{Remark}
\theoremstyle{plain}

\title{The Dynamics of Continuous Mixture Collapse in Language Models}

%% file: paper_body.tex
\maketitle

\begin{abstract}
LLMs latent-state reasoning methods replace discrete intermediate tokens with continuous states, such as weighted mixtures of token embeddings, to retain multiple possible reasoning directions rather than committing to one. Yet pretrained language models often fail to preserve these mixtures. We study \emph{why} through a combination of theoretical analysis and controlled empirical investigations on a variety of models. We identify three independent, distinct sources of failure. First, transformer architectures already distort mixture geometry, and training substantially amplifies this effect. Moreover, the failure can occur even if the model transports mixtures perfectly linearly: the softmax readout and autoregressive feedback form a dynamical system that either amplifies small differences until one component of the mixture dominates or contracts different mixtures until they become indistinguishable. We verify this theoretical prediction empirically: the observed transition between contraction and amplification occurs near the theoretical threshold derived by our analysis, and pretrained-model rollouts lie predominantly on the amplifying side. Finally, we generalize to mixtures of many components and show that exact preservation generally requires context-dependent correction, whose required dimensionality can grow with the number of components.
\end{abstract}

\section{Introduction}
\label{sec:intro}

Large language models have achieved remarkable results across a wide range of complex reasoning tasks.
A key technique behind that success is chain-of-thought (CoT)
reasoning~\cite{nye2021workscratchpadsintermediatecomputation,wei2023chainofthoughtpromptingelicitsreasoning,kojima2023largelanguagemodelszeroshot,wang2023selfconsistencyimproveschainthought}, which has the
model work through a problem step by step by generating intermediate reasoning steps in natural
language.
At each step, the model computes a distribution over the vocabulary but commits to a single token, discarding the information carried by all other candidates and preventing it from revisiting them. This commitment makes generation \emph{readable}, but intermediate reasoning need not be written for human consumption. Forcing every step through this discrete linguistic bottleneck may therefore limit the abstract concepts the model can represent and manipulate.

Building on that, several recent works replace discrete intermediate reasoning tokens with continuous reasoning states. \emph{Soft Thinking} constructs each reasoning state as a probability-weighted mixture of vocabulary embeddings,
$
e(p)=\sum_i p_i e_i,
$
and feeds this continuous representation back into the model at the next reasoning step~\cite{zhang2025softthinkingunlockingreasoning}. Other continuous-reasoning methods use different representations: \emph{Coconut} directly feeds the model's last hidden state back as the next input embedding~\cite{hao2026traininglargelanguagemodels}, while \emph{CODI} learns a continuous chain of thought by distilling explicit chain-of-thought reasoning into continuous hidden states~\cite{shen2025codicompressingchainofthoughtcontinuous}. \emph{Mixture of Inputs} keeps the discarded distribution without any training, feeding back a Bayesian posterior over the sampled token and the distribution it came from~\cite{zhuang2025textgenerationdiscretetoken}.

Continuous reasoning is appealing for two main reasons. First, it can be more \textbf{efficient}: continuous states can compress reasoning that would otherwise require many discrete tokens so they can represent several reasoning traces within the same computation~\cite{hao2026traininglargelanguagemodels,shen2025codicompressingchainofthoughtcontinuous,gozeten2026continuouschainthoughtenables}. Second, it enables broader \textbf{exploration}: rather than committing to a single intermediate token, a continuous state can retain multiple possible reasoning directions and propagate them in parallel~\cite{zhang2025softthinkingunlockingreasoning,hao2026traininglargelanguagemodels,gozeten2026continuouschainthoughtenables}. This intuition also has theoretical support: continuous chain-of-thought allows a two-layer transformer to solve graph reachability by maintaining a continuous state of reasoning traces, and later work shows that such continuous states can emerge through gradient-based training~\cite{zhu2025reasoningsuperpositiontheoreticalperspective,zhu2026emergencesuperpositionunveilingtraining}.

However, recent work finds that pretrained models do not reliably preserve these continuous states. Wu et al.\ find that Soft Thinking decoding is dominated by the highest-probability component of the reasoning state, creating a greedy feedback loop that suppresses alternative reasoning paths~\cite{wu2025llmssinglethreadedreasonersdemystifying}. Rizvi-Martel et al.\ similarly find that, in training-free and fine-tuned latent-reasoning settings, continuous state either collapses or is not used, while signs of continuous state appear only in models trained from scratch with latent thoughts~\cite{rizvimartel2026illusionsuperpositionprincipledanalysis}.

In this work, we ask why. We show that collapse has architectural, learned, and dynamical sources, each capable of destroying the mixture on its own. Comparing pretrained models with matched randomly initialized controls across eight models separates distortion induced by the architecture from that added by training. We then show that this problem persists even under perfect linear transport. The softmax and feedback between successive generation steps can themselves destroy the mixture: when the two continuous reasoning components lead to sufficiently different next-step preferences, small imbalances are amplified until one component dominates; when their preferences are similar, different mixtures instead converge and become indistinguishable. Finally, extending the analysis to \(K\) components, we show that preserving a mixture generally requires a correction that adapts to the context, and that the amount of context-dependent information required can grow with the number of components under plausible conditions.

\section{Where Continuous Mixtures Lose Their Geometry}
\label{sec:phenomenon}

Let $e_i\in\mathbb{R}^d$ be the embedding of token $i$. For $p$ on the simplex, the injected continuous state
is $e(p)=\sum_i p_i e_i$. We place it at an interior slot and ask a question after that slot, so the
model must use the injected state to answer. Write $q_i$ for the next-token distribution when component
$i$ occupies the slot as an ordinary token and $q(p)$ for the distribution under the continuous state.

\begin{figure}[t]
\centering
\includegraphics[width=\linewidth]{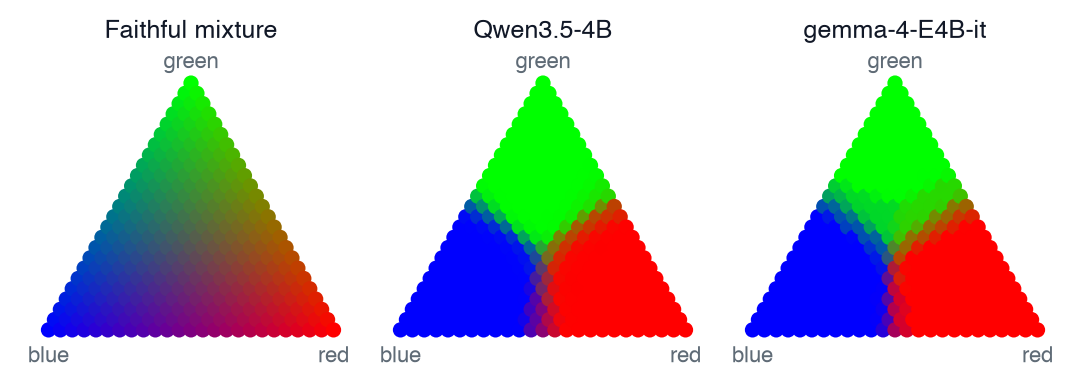}
\caption{\textbf{Three-way mixtures collapse toward individual components.}
Each point represents an injected convex combination of three token embeddings, positioned by its requested mixture weights and colored by the model's probabilities for \textit{red}, \textit{green}, and \textit{blue}. Faithful preservation would reproduce the smooth reference simplex on the left. Qwen3.5-4B (center) and Gemma-4-E4B-it (right) instead map most interior mixtures to near-pure colors, with narrow transition boundaries between components.}
\label{fig:simplex}
\end{figure}

As a motivation, take three words whose colors are obvious and
different (e.g., \emph{blood}, \emph{grass}, \emph{sky}) and build one
vector that is half the first and a quarter each of the other two for instance. Put that single vector in the slot
and ask \emph{What color is \textbf{[\,\_\,]}? Answer with a single word.} If the model treated the
blend as a blend, the answer should be $50\%$ red, $25\%$ green and $25\%$ blue. Averaged over the six
triples and ten rewordings of the question, however, the next token prediction was $99.5\%$ red. Figure~\ref{fig:simplex} repeats this across a lattice of 325 three-way mixtures. In both models, almost the entire interior of the simplex is mapped to one near-pure component, with abrupt transitions across narrow boundaries.

To separate distortion caused by the transformer architecture itself from distortion induced by
training, we compare each pretrained model with five randomly initialized networks of exactly the
same architecture. The untrained models therefore retain the same nonlinear blocks, normalization,
attention structure, and depth, but remove the effect of the learned weights. We run this experiment on a ladder of models for both Qwen3.5~\cite{qwen35blog} and Gemma~4~\cite{gemmateam2026gemma4technicalreport}.

To measure how well a mixture is preserved, let $h(w) \in \mathbb{R}^d$ be the hidden state produced
by an injected mixture with weight $w$. We define the \textbf{recovered mixture weight}
as

\begin{equation}
\alpha(w)
=
\frac{
\langle h(w)-h(0),\, h(1)-h(0)\rangle
}{
\lVert h(1)-h(0)\rVert^2
}.
\label{eq:alpha}
\end{equation}

This measures where $h(w)$ lies between the two components' hidden states.
Exact mixture preservation gives $\alpha(w)=w$.

We summarize preservation by averaging the absolute difference between the
requested weight $w$ and the recovered weight $\alpha(w)$:

\begin{equation}
\mathrm{CALIB}
=
1-\frac{1}{Z}
\mathbb{E}_{w}
\left[
|\alpha(w)-w|
\right].
\label{eq:calib}
\end{equation}

Here, $Z$ scales the score so that exact preservation gives $\mathrm{CALIB}=1$,
while a hard threshold at $w=0.5$ gives $\mathrm{CALIB}=0$.

\begin{figure}[t]
\centering
\includegraphics[width=\linewidth]{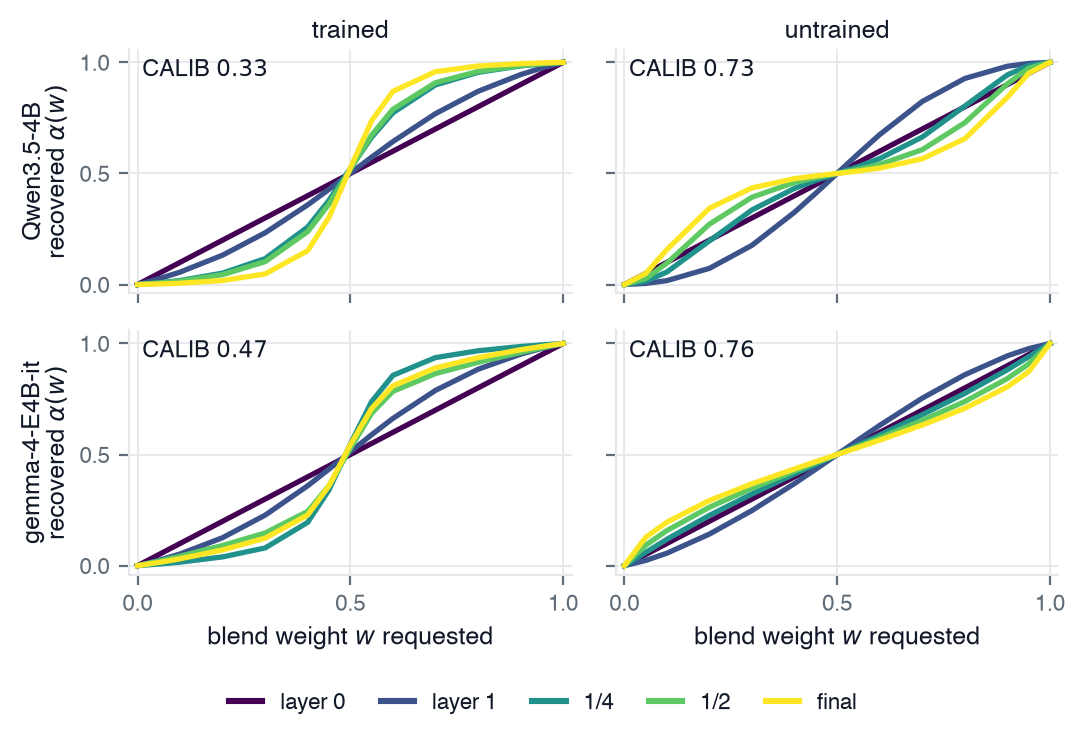}
\caption{\textbf{Training drives mixture collapse through depth.}
Recovered mixture weight $\alpha(w)$ as a function of the requested mixture weight $w$ at five normalized depths. Rows show Qwen3.5-4B and Gemma-4-E4B-it; columns show the trained model and the mean over five randomly initialized controls of the same architecture. The diagonal corresponds to exact mixture preservation. Trained models develop increasingly step-like responses through depth, while the matched controls remain closer to the diagonal. Panel titles report final-layer CALIB, corresponding to Table~\ref{tab:main}. Results use the same 1000-item benchmark described in Appendix~\ref{app:bench}.}

\label{fig:layers}
\end{figure}

\begin{table}[t]
\caption{\textbf{Mixture preservation in trained models and matched untrained controls.}
CALIB (Eq.~\ref{eq:calib}) scores exact preservation of the injected mixture as $1$ and a hard threshold at $w=0.5$ as $0$. Left: CALIB across depth for Qwen3.5-4B and Gemma-4-E4B-it, corresponding to Figure~\ref{fig:layers}. Right: final-layer CALIB across a broader ladder of eight models. Each untrained value is the mean over five random initializations of the same architecture, with standard deviation reported alongside it.}
\label{tab:main}

\sbox{\tabcalib}{\renewcommand{\arraystretch}{1.1837}\input{table_calib}}
\sbox{\tabladder}{\input{table_ladder_mini}}
\par\vspace{\belowcaptionskip}
\noindent\vbox{\hbox to \linewidth{\hfil\usebox{\tabcalib}\hfil\usebox{\tabladder}\hfil}}
\end{table}
Figure~\ref{fig:layers} and Table~\ref{tab:main} show a clear separation between trained models and matched untrained controls. With depth, trained networks increasingly distort the injected mixture, whereas the controls remain substantially closer to linear interpolation. The two families exhibit different depthwise profiles: Qwen accumulates collapse toward the output, while Gemma shows stronger non-monotonic distortion within the stack. Across the broader model ladder, trained models consistently preserve mixtures less faithfully than their controls, and this gap generally widens with scale. Thus, architectural effects alone induce some distortion, but training substantially amplifies it.

\section{Mixture Collapse as a Dynamical System}
\label{sec:stability}

The previous section identified both architectural and training-induced distortion of the mixture.
A natural question is whether removing these effects would be enough to preserve the mixture. We
show that it would not: even if the transformer propagated the mixture \emph{perfectly linearly},
the softmax readout and its autoregressive feedback introduce a separate source of collapse.

Fix a context and let two components $A$ and $B$ produce next-token distributions $q_A$ and $q_B$.
Suppose, optimistically, that their mixed final hidden state is preserved exactly linearly. Since the
output head is linear, the corresponding logits satisfy
\[
z(w)=wz_A+(1-w)z_B.
\]
After applying the softmax,
\begin{equation}
q_w(y)\propto q_A(y)^w q_B(y)^{1-w},
\label{eq:geom}
\end{equation}
so linear interpolation in logit space becomes a weighted \emph{geometric} mixture in probability
space, rather than the arithmetic mixture $wq_A+(1-w)q_B$. Autoregressive decoding then feeds this
distorted mixture back into the model, turning the effect into a recursive dynamical system.

\subsection{Recursive Dynamics of Mixture Propagation}
\label{sec:recursion}

Here, for simplicity, we restrict the vocabulary to two tokens $A$ and $B$, and track
how probability mass moves between them over the rollout.
At step $t$, let $w_t\in[0,1]$ be the mixture weight on $A$, so that $1-w_t$ is the weight on
$B$. It is convenient to recenter this quantity as
\begin{equation}
u_t = 2w_t-1,
\label{eq:u}
\end{equation}
so that $u_t=-1$ corresponds to pure $B$, $u_t=0$ to an equal mixture, and $u_t=+1$ to pure
$A$. The sign of $u_t$ therefore indicates which token has the majority, while $|u_t|$ measures
how strongly the mixture favors one token over the other.

At the current context $c_t$, we evaluate the model $M$ separately on the two pure components,
\[
q^A_t = M(c_t\Vert A),
\qquad
q^B_t = M(c_t\Vert B).
\]
Each branch's log-odds between that pair is
\begin{equation}
\Delta_{A,t}=\log\frac{q^A_t(A)}{q^A_t(B)},
\qquad
\Delta_{B,t}=\log\frac{q^B_t(A)}{q^B_t(B)},
\label{eq:deltas}
\end{equation}
so $\Delta_{A,t}$ measures how strongly branch $A$ prefers its own successor to $B$'s. We write
\begin{equation}
L_t=\frac{\Delta_{A,t}-\Delta_{B,t}}{2},
\qquad
b_t=\frac{\Delta_{A,t}+\Delta_{B,t}}{2}.
\label{eq:Lb}
\end{equation}

$L_t$ is the \emph{coupling}: it measures how differently the two branches rank the competing
descendants. Large $|L_t|$ means their relative preferences differ strongly, while small $|L_t|$
means they make similar relative predictions. In contrast, $b_t$ is the \emph{field}: it measures the common offset in the two branches' log-odds.

Under exact linear logit transport, the next log-odds is
$\ell_{t+1}=w_t\Delta_{A,t}+(1-w_t)\Delta_{B,t}=b_t+L_tu_t$. Using
$2\sigma(x)-1=\tanh(x/2)$ gives
\begin{equation}
u_{t+1}=\tanh\!\left(\frac{b_t+L_tu_t}{2}\right).
\label{eq:cw}
\end{equation}

For fixed $b$ and $L$, Eq.~\ref{eq:cw} is the standard fixed-point iteration associated with the
mean-field Curie--Weiss self-consistency equation~\cite{weiss:jpa-00241247}. The language-model setting is
non-autonomous: $L_t$ and $b_t$ vary as the rollout evolves.

\subsection{The Critical Threshold: Amplification and Washout}

\label{sec:threshold}
We now study what happens to the mixture over time under these dynamics.
To isolate the effect of the coupling, we first set $b_t=0$ and assume
$L_t>0$. Equation~\ref{eq:cw} then becomes

\begin{equation}
u_{t+1}=\tanh\!\left(\frac{L_tu_t}{2}\right).
\label{eq:symmetric-dynamics}
\end{equation}

Take first the constant-coupling case $L_t=L$. The balanced state $u=0$ is always a fixed point,
but its stability changes at $L=2$. Indeed,

\[
\left.\frac{\partial u_{t+1}}{\partial u_t}\right|_{u_t=0}
=
\frac{L}{2}.
\]

Hence $u=0$ is stable for $L<2$ and unstable for $L>2$. When $L>2$, two additional stable
fixed points $\pm u_\ast(L)$ appear, where

\begin{equation}
u_\ast(L)=\tanh\!\left(\frac{Lu_\ast(L)}{2}\right).
\label{eq:cw-fixed-point}
\end{equation}

Starting from any $u_0\neq0$, the sign of the initial majority is preserved and the trajectory
converges to the corresponding polarized fixed point. This is the classical bifurcation of the
mean-field ferromagnet~\cite[Ch.~2]{Friedli_Velenik_2017}. The polarization rapidly approaches a pure state
as $L$ grows: for example, $L=4$ gives $u_\ast\simeq0.96$, while $L=8$ gives
$u_\ast\simeq0.9993$, and $u_\ast(L)\to1$ as $L\to\infty$.

However, the LLM setting is harder than the classical Curie--Weiss case because the context evolves,
so $L_t$ is not fixed. The following two results characterize the corresponding time-varying
dynamics on either side of the critical value.

\begin{theorem}[Uniform coupling implies persistent polarization]
\label{thm:polar}

Let $u_0\neq0$, let $b_t=0$, and suppose $L_t\ge L_{\min}>2$ for every $t$. Let
$u_\ast(L_{\min})\in(0,1)$ be the positive solution of

\[
u=\tanh\!\left(\frac{L_{\min}u}{2}\right).
\]

Then $\operatorname{sign}(u_t)=\operatorname{sign}(u_0)$ for all $t$, and

\begin{equation}
\liminf_{t\to\infty}|u_t|
\ge
u_\ast(L_{\min}).
\label{eq:floor}
\end{equation}

\end{theorem}

\begin{theorem}[Subcritical coupling washes out the mixture]
\label{thm:washout}

Suppose

\[
0\le L_t\le L_{\max}<2
\]

along every rollout. Then

\[
|u_T|
\le
\left(\frac{L_{\max}}{2}\right)^T |u_0|
+
\frac12
\sum_{t=0}^{T-1}
\left(\frac{L_{\max}}{2}\right)^{T-1-t}|b_t|.
\]

Consequently, if $b_t = 0$, then $u_t\to0$ independently of the initial mixture.

\end{theorem}

The proofs are given in Appendix~\ref{app:theory}. Theorem~\ref{thm:polar} follows by a
monotonicity comparison, while Theorem~\ref{thm:washout} follows from a mean-value bound on the
same recursion.

Together, the two results describe complementary failure modes on either side of the critical value.
If the evolving context keeps the coupling uniformly above the critical value, the initial majority
is magnified and the mixture polarizes to an almost pure state for large $L_{\min}$. Below the
critical value, the dynamics instead become contractive and progressively erase distinctions between
different mixtures. In particular, avoiding a vertex is not enough to preserve the mixture: distinct
states such as $70/30$ and $30/70$ become indistinguishable at a geometric rate, and the information
carried by the mixture is lost. Thus the two sides of the threshold destroy information in
complementary ways.

\begin{figure}[t]
\centering
\includegraphics[width=\linewidth]{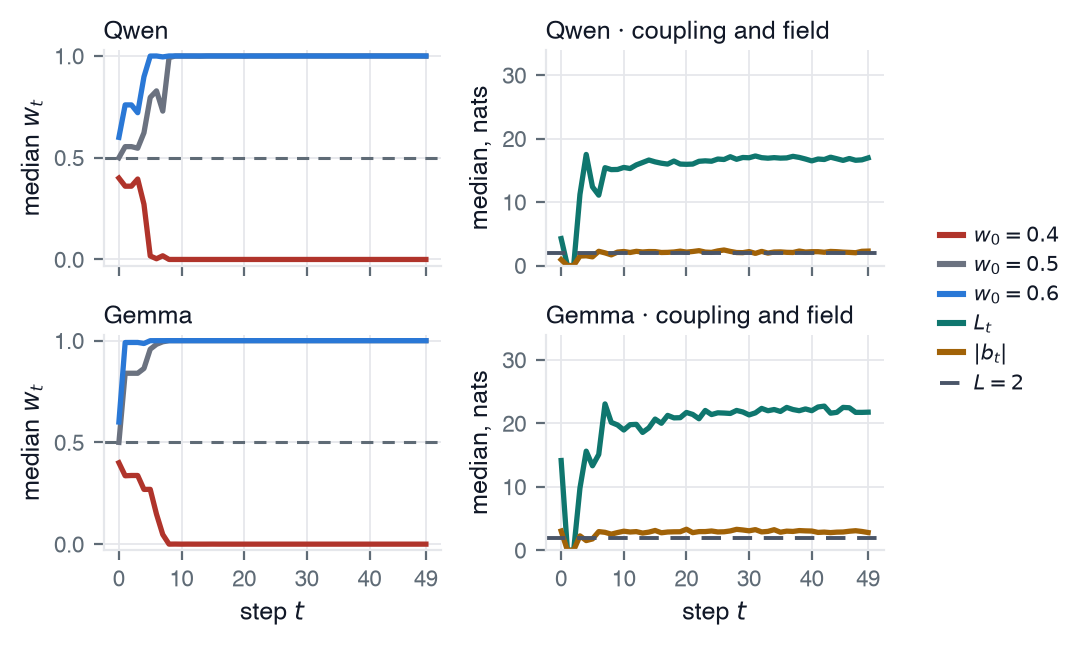}
\caption{\textbf{Recursive softmax reproduces polarization in LLM rollouts.}
The two pure branches evolve under the model, while their mixture is formed only by linearly
interpolating the endpoint logits before the softmax. Rows are the two families, Qwen3.5-4B above
and Gemma-4-E4B-it below. The left column tracks the median mixture weight $w_t$ from each initial
mixture $w_0$; the right column shows both measured coefficients of the recursion along the same
rollouts, the coupling $L_t$ and the field magnitude $|b_t|$. Every curve is a median over the same
$1000$-item benchmark used in Figure~\ref{fig:layers} and described in
Appendix~\ref{app:bench}, rolled out for $50$ steps per item.}
\label{fig:branch}
\end{figure}

\begin{figure}[t]
\centering
\includegraphics[width=\linewidth]{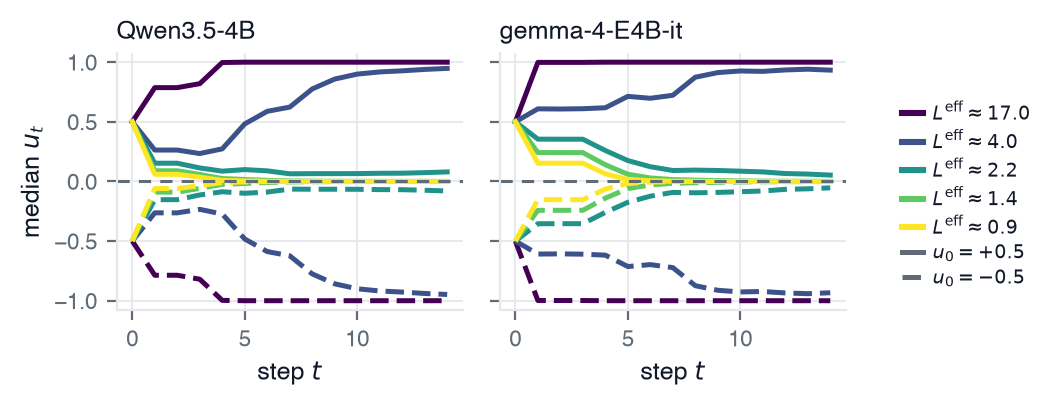}
\caption{\textbf{Reducing the gain replaces one failure mode with the other, in both families.}
A feedback temperature $\tau$ divides the effective coupling, $L^{\mathrm{eff}}_t=L_t/\tau$, so sweeping $\tau$ walks the same measured items from the polarizing regime $L^{\mathrm{eff}}_t>2$ into the contractive one $L^{\mathrm{eff}}_t<2$. Each panel shows the median $u_t$ from two symmetric starts $u_0=\pm0.5$, one panel per family, with color carrying the \emph{median} effective coupling each curve achieves. Well above threshold the two starts separate to opposite vertices; well below it they merge onto the balanced state, and the turn happens around the predicted $L^{\mathrm{eff}}_t=2$. The couplings are the $L_t$ sequences measured in Figure~\ref{fig:branch}, over the same $1000$-item benchmark of Appendix~\ref{app:bench}.}
\label{fig:gain}
\end{figure}

\section{Recursive Softmax in Pretrained Language Models}
\label{sec:llm-dynamics}

The theory above predicts two qualitatively different failure modes depending on the coupling. We
now measure these dynamics directly in pretrained language-model trajectories. To isolate the effect of the recursive
softmax, we run the model only on the two pure branches and form their mixture directly in logit
space, so that mixture transport is exactly linear by construction and the only nonlinearity acting on
the mixture is the softmax.

Following the notation of \S\ref{sec:recursion}, the state at step $t$ is
$(A_t,B_t,w_t,c_t)$, and we evaluate the model on the two pure branches as before,
$q^A_t=M(c_t\Vert A_t)$ and $q^B_t=M(c_t\Vert B_t)$, using their greedy predictions as the
next pair of components,
\begin{equation}
A_{t+1}=\arg\max_y q^A_t(y),
\qquad
B_{t+1}=\arg\max_y q^B_t(y).
\label{eq:branches}
\end{equation}

The two branches therefore evolve normally with the context. We then interpolate only their endpoint
logits,
\[
z_t^{\mathrm{mix}}
=
w_tz_t^A+(1-w_t)z_t^B,
\]
apply the softmax, and define the next mixture weight from the probability assigned by the model to the two new
components under the mixed logits,
\[
w_{t+1}
=
\frac{p_t(A_{t+1})}
{p_t(A_{t+1})+p_t(B_{t+1})}.
\]
and $c_{t+1}$ is the context updated with the new mixture.
Because the endpoints and context evolve at every step, the corresponding $L_t$ and $b_t$ are
re-measured throughout the rollout.

Figure~\ref{fig:branch} tests this prediction on pretrained model trajectories while removing
nonlinear transformer processing from the mixture path. The left panel shows the resulting evolution
of the mixture weight $w_t$: mixtures initialized just above and below $1/2$ rapidly diverge toward
opposite branches. The right panel explains this behavior through the two measured coefficients of
Eq.~\ref{eq:cw}. Across both model families, $L_t$ remains predominantly above the critical value
$L=2$, placing the dynamics in the majority-amplifying regime of \S\ref{sec:threshold}, while $b_t$
is much smaller.

\subsection{Crossing the Critical Threshold}
\label{sec:gain-sweep}

The two regimes on both sides of the critical threshold can be walked between on the measured dynamics. A temperature $\tau$ on the
feedback logits rescales the coupling as $L_t^{\mathrm{eff}}=L_t/\tau$, varying the gain without changing
the model, the items or the branch trajectories, so we sweep $\tau$ over the couplings measured in
Figure~\ref{fig:branch} and start the recursion from two symmetric states $u_0=\pm0.5$.

Figure~\ref{fig:gain} shows the predicted transition in both families: at large effective coupling the
two trajectories separate to opposite branches; at small effective coupling they merge and lose their
initial distinction, and the turn happens around the theoretical value $L^{\mathrm{eff}}=2$ which is exactly what the theory predicts. 
\section{The Cost of Correcting \texorpdfstring{\(K\)}{K}-Way Mixtures}
\label{sec:context}

The analysis above is the special case of mixtures of two tokens. We now generalize to $K$ components. Let
$p\in\operatorname{int}\Delta^{K-1}$ denote their mixture weights and use
log-ratio coordinates
\begin{equation}
    a=\psi(p)\in\mathbb{R}^{K-1},
    \qquad
    \psi_i(p)=\log\frac{p_i}{p_K}.
    \label{eq:logratio}
\end{equation}

Although the textual context $c_t$ is discrete, the model operates on a continuous
representation of it. Let $x_t=x(c_t)\in\mathbb{R}^n$ denote this contextual
representation. One decoding step then induces a map
\begin{equation}
    a_{t+1}=\Phi(x_t,a_t),
    \label{eq:phi}
\end{equation}
where the contextual representation and candidate identities may change during the rollout;
the binary linear-logit reduction recovers Eq.~\ref{eq:cw}.

If $a_t$ represents the proportions of unresolved components, the ideal behavior
is that one reasoning step leaves them unchanged,
\begin{equation}
    a_{t+1}=a_t.
    \label{eq:faithful}
\end{equation}

The previous sections establish, theoretically and empirically, that in general
$\Phi(x_t,a_t)\neq a_t$. Moreover, the distortion depends on the context, so a
fixed transformation of $a_t$ cannot in general undo it. Suppose instead that,
before feeding the mixture into the model, we apply a corrector
\begin{equation}
    \widetilde a_t
    =
    f\!\left(a_t,r(x_t)\right),
    \label{eq:corrector}
\end{equation}
where $r(x_t)\in\mathbb{R}^m$ is context-dependent information supplied to the
corrector. Faithful propagation requires
\begin{equation}
    \Phi\!\left(
        x_t,
        f(a_t,r(x_t))
    \right)
    =
    a_t.
    \label{eq:corrected-faithful}
\end{equation}

We now ask how much context-dependent information such a corrector may require.
Fix a single interior target state $\bar a$ and require only that the corrector
keep this one state fixed under nearby perturbations of the continuous contextual
representation:
\begin{equation}
    \Phi\!\left(
        x,
        f(\bar a,r(x))
    \right)
    =
    \bar a
    \qquad
    \text{for every }x\text{ near }x_0.
    \label{eq:anchor}
\end{equation}

Even this weak requirement can demand a number of context-dependent quantities
proportional to the number of components.

\begin{theorem}[Context-dimension lower bound]
\label{thm:drift}

Assume that $\Phi$, $f$, and $r$ are continuously differentiable and that
Eq.~\ref{eq:anchor} holds. Then
\begin{equation}
    m
    \ge
    \operatorname{rank}
    D_x\Phi\!\left(
        x_0,
        f(\bar a,r(x_0))
    \right).
    \label{eq:context-rank}
\end{equation}

In particular, if local variations of the continuous contextual representation
move the output independently in all $K-1$ mixture directions, then
\begin{equation}
    m\ge K-1.
\end{equation}
\end{theorem}

The proof is a chain-rule argument on the anchoring condition
(Appendix~\ref{app:theory}). The rank of this Jacobian need not be small: local changes in the contextual representation can affect the final hidden state along many independent directions. As a result, faithfully correcting a mixture as it passes through a pretrained model generally requires a non-trivial dependence on the context, and the amount of context-dependent information needed can itself be large.
\section{Discussion and Implications}
\label{sec:discussion}

Taken together, our experiments and theory give a unified account of why continuous mixtures are difficult
to preserve. For each failure mode, we identify a corresponding mechanism and test its consequences:
the transformer architecture already distorts mixtures, training substantially amplifies this distortion,
and even under idealized linear transport the softmax feedback dynamics can either amplify or erase
differences between components. The rollout experiments reproduce the behavior predicted by this
dynamical analysis, suggesting that these are not merely artifacts of the representation-level probes.
This also explains why a simple fix such as globally reweighting or softening the mixture is unlikely to
be sufficient. The required correction changes with the context, and our $K$-way analysis shows that,
under full-rank context variation, exact local preservation can require at least $K-1$ context-dependent
degrees of freedom. Thus, preserving mixtures is not just a matter of choosing better weights, but of
actively compensating for context-dependent dynamics throughout the rollout.

\paragraph{Relation to prior work.}
Continuous- and latent-reasoning methods show that continuous computation can be useful and, trained appropriately, can carry multiple reasoning traces~\cite{hao2026traininglargelanguagemodels,shen2025codicompressingchainofthoughtcontinuous,zhang2025softthinkingunlockingreasoning,zhu2025reasoningsuperpositiontheoreticalperspective,zhu2026emergencesuperpositionunveilingtraining,gozeten2026continuouschainthoughtenables}, while recent empirical work finds pretrained models collapsing or ignoring such continuous states~\cite{wu2025llmssinglethreadedreasonersdemystifying,rizvimartel2026illusionsuperpositionprincipledanalysis}. Our results connect the two: continuity alone does not preserve continuous states, and what decides the outcome is whether the representation, the readout and the closed-loop dynamics together preserve the semantics assigned to the continuous state. The positive and negative results describe different dynamical regimes of one problem. Training-free remedies already target the failure from this direction. \emph{SeLaR} gates soft embeddings on entropy and adds a contrastive term that pushes the blend away from the dominant token~\cite{fu2026selarselectivelatentreasoning}, and \emph{Mixture of Inputs} reweights what is fed back~\cite{zhuang2025textgenerationdiscretetoken}, which in the language of \S\ref{sec:threshold} is an attempt to hold the effective coupling down without falling into the contractive regime, and Theorem~\ref{thm:washout} says that is the trade such a term has to manage.

\paragraph{Future directions.}
Each cause suggests its own intervention. Architectural and training-induced distortion could be addressed by training explicitly on mixed or latent states, so that mixture geometry survives the network~\cite{hao2026traininglargelanguagemodels,shen2025codicompressingchainofthoughtcontinuous,butt2025softtokenshardtruths,zheng2026softgrposurpassingdiscretetokenllm}; supervising a latent directly as an average of the embeddings it is meant to stand for, or confining it to the vocabulary space, are two ways of doing that already in
use~\cite{yerram2026trainingcontinuouschainthought,deng2026llmlatentreasoningchain}.
The recursive softmax failure instead suggests bypassing the vocabulary readout altogether:
rather than projecting a hidden state to logits, one can feed the continuous hidden state directly
into the next step. Coconut takes precisely this route, using the previous final hidden state as the
next input embedding~\cite{hao2026traininglargelanguagemodels}; its gains are nevertheless task-dependent, suggesting that
removing the softmax feedback loop addresses one source of collapse but is not by itself sufficient
for robust continuous reasoning.
And because the required correction moves with the context, a general $K$-way solution likely needs a context-conditioned controller whose capacity grows with the number of components. Learning and testing such neutral dynamics is the natural next step toward systems that preserve, rather than merely encode, multiple possibilities.

\section{Limitations and Broader Impact}
Our experiments use controlled embedding mixtures as an operational notion of preserving multiple
components; this may not capture every form of semantic continuous state in continuous-reasoning
methods. The empirical analysis is primarily binary and covers two model families, while the $K$-way
result is theoretical and its full-rank condition is not measured directly. Matched random controls
isolate effects associated with learned weights but do not identify which aspects of training cause the
distortion, and we study mixture preservation rather than downstream performance in a fully trained
continuous-reasoning system.

This work is primarily diagnostic. Better understanding and controlling mixture collapse may improve
the reliability of latent-reasoning systems, but stronger latent reasoning can also make intermediate
computation less human-readable. This increases the importance of developing monitoring and
interpretability methods alongside improvements in continuous reasoning.

\begin{ack}
We thank Maxwell Sun, Beshr Islam Bouli, Seb Losada and Nour Massri for helpful discussions and feedback.
\end{ack}

{\small
\bibliographystyle{plain}
\bibliography{references}
}

\appendix

\input{appendix_bench}

\input{appendix_theory}

%% file: table_calib.tex
\centering
\scriptsize
\setlength{\tabcolsep}{2pt}
\begin{tabular}{lcccc}
\toprule
& \multicolumn{2}{c}{Qwen3.5-4B} & \multicolumn{2}{c}{gemma-4-E4B-it} \\
\cmidrule(lr){2-3}\cmidrule(lr){4-5}
Depth & trained & untr. & trained & untr. \\
\midrule
$0\%$ (control) & 1.000 & $1.000_{\pm0.000}$ & 1.000 & $1.000_{\pm0.000}$ \\
layer 1 & \textbf{0.820} & $\mathbf{0.696_{\pm0.002}}$ & \textbf{0.796} & $\mathbf{0.799_{\pm0.001}}$ \\
$12.5\%$ & 0.678 & $0.837_{\pm0.002}$ & 0.651 & $0.854_{\pm0.002}$ \\
$25\%$ & 0.507 & $0.873_{\pm0.001}$ & 0.385 & $0.860_{\pm0.002}$ \\
$50\%$ & 0.464 & $0.824_{\pm0.002}$ & 0.539 & $0.825_{\pm0.003}$ \\
$75\%$ & 0.391 & $0.774_{\pm0.002}$ & 0.461 & $0.790_{\pm0.004}$ \\
$100\%$ (final) & \textbf{0.332} & $\mathbf{0.732_{\pm0.003}}$ & \textbf{0.475} & $\mathbf{0.765_{\pm0.003}}$ \\
\bottomrule
\end{tabular}

%% file: table_ladder_mini.tex
\centering
\scriptsize
\setlength{\tabcolsep}{3pt}
\begin{tabular}{lcc}
\toprule
& \multicolumn{2}{c}{final CALIB} \\
\cmidrule(lr){2-3}
Model & trained & untr. \\
\midrule
Qwen3.5-0.8B & \textbf{0.632} & $0.818_{\pm0.002}$ \\
Qwen3.5-2B & \textbf{0.506} & $0.780_{\pm0.004}$ \\
Qwen3.5-4B & \textbf{0.332} & $0.732_{\pm0.003}$ \\
Qwen3.5-9B & \textbf{0.249} & $0.730_{\pm0.005}$ \\
Qwen3.5-27B & \textbf{0.130} & $0.611_{\pm0.004}$ \\
\midrule
gemma-4-E4B-it & \textbf{0.475} & $0.765_{\pm0.003}$ \\
gemma-4-12B-it & \textbf{0.227} & $0.581_{\pm0.005}$ \\
gemma-4-31B-it & \textbf{0.160} & $0.569_{\pm0.003}$ \\
\bottomrule
\end{tabular}

%% file: appendix_bench.tex
\section{Benchmark Construction}
\label{app:bench}

We construct a benchmark of 1000 binary mixture items spanning seven semantic categories:
color, temperature, size, speed, hardness, weight, and brightness. Each item consists of a prompt
with one injection slot, two component words, and two corresponding answer words. At mixture
weight $p$, the slot is filled with
\[
e(p)=p\,e_1+(1-p)\,e_2,
\]
where $e_1$ and $e_2$ are the component embeddings. We then measure the model's relative
probability assigned to the two answer words at the next-token position. Thus, $p=1$ and $p=0$
recover the two ordinary-token endpoints, while intermediate $p$ values test whether the model
preserves the injected mixture.

\paragraph{Item selection.}
An item is included only if (i) both component words and both answer words are single tokens
under both model families, and (ii) at each endpoint the trained model assigns at least $0.5$
probability to the corresponding answer word. The second condition ensures that the underlying
question is well posed before testing mixtures. The surviving items are intersected across model
families, yielding one shared 1000-item benchmark used for all trained models and their matched
randomly initialized controls.

\paragraph{Prompting.}
We use a chat template that closes the reasoning block before the answer position, so that the
next-token probability is concentrated on the requested one-word answer rather than on intermediate
reasoning text.

Table~\ref{tab:examples} shows representative benchmark items.

\input{table_examples}

%% file: table_examples.tex
\begin{table}[t]
\caption{\textbf{Ten of the 1000 benchmark items.} One per category, then a second pass
in the same order; every category the benchmark contains appears before any
repeats. \textbf{Prompt} is the template with the injection slot marked
\textbf{[\,\_\,]}; the slot is filled with the convex combination
$p\,e(\text{first})+(1-p)\,e(\text{second})$ of the two \textbf{components}'
embeddings, never with a real token. \textbf{Answers} are the two words whose
probability ratio is the readout: with $p=1$ the model should say the first, with
$p=0$ the second, and the question is what it does in between. Every item is
single-token under both tokenizers and clears the $0.5$ answer-mass gate in both
families.}
\label{tab:examples}
\centering
\footnotesize
\setlength{\tabcolsep}{4pt}
\begin{tabular}{@{}p{1.6cm}p{6.9cm}p{2.5cm}p{2.0cm}@{}}
\toprule
Category & Prompt & Components & Answers \\
\midrule
brightness & \texttt{State whether \textbf{[\,\_\,]} is bright or dark, in a single word.} & sun / pit & bright / dark \\
color & \texttt{Name the color of \textbf{[\,\_\,]}. One word only.} & coal / sky & black / blue \\
hardness & \texttt{ \textbf{[\,\_\,]} is hard or soft? Reply with one word.} & granite / cotton & hard / soft \\
size & \texttt{State whether \textbf{[\,\_\,]} is big or small, in a single word.} & whale / atom & big / small \\
speed & \texttt{In one word, is \textbf{[\,\_\,]} fast or slow?} & rocket / turtle & fast / slow \\
temperature & \texttt{ \textbf{[\,\_\,]} is hot or cold? Reply with one word.} & lava / fridge & hot / cold \\
weight & \texttt{Is \textbf{[\,\_\,]} heavy or light? Answer with a single word.} & lead / feather & heavy / light \\
brightness & \texttt{Is \textbf{[\,\_\,]} bright or dark? Answer with a single word.} & sun / void & bright / dark \\
color & \texttt{In one word, what color is \textbf{[\,\_\,]}?} & ink / ocean & black / blue \\
hardness & \texttt{Give the hardness of \textbf{[\,\_\,]} using only one word.} & granite / sponge & hard / soft \\
\bottomrule
\end{tabular}
\end{table}

%% file: appendix_theory.tex
\section{Proofs}
\label{app:theory}

\subsection*{Proof of Theorem~\ref{thm:polar} (uniform coupling implies persistent polarization)}

\begin{proof}
Write $F_L(u)=\tanh(Lu/2)$. For $u>0$ it is increasing in $u$ and in $L$, and $F_L(-u)=-F_L(u)$, so
without loss of generality take $u_0>0$. Since every $F_{L_t}$ maps $(0,1)$ into itself, $u_t>0$ for all
$t$ and the sign is preserved.

Put $v_0=u_0$ and $v_{t+1}=F_{L_{\min}}(v_t)$, the constant-coupling system at the smallest coupling
encountered. If $u_t\ge v_t$ then $u_{t+1}=F_{L_t}(u_t)\ge F_{L_{\min}}(u_t)\ge F_{L_{\min}}(v_t)=v_{t+1}$,
using monotonicity in $L$ then in $u$; by induction $u_t\ge v_t$ for every $t$.

For $L_{\min}>2$ the map $F_{L_{\min}}$ has exactly one positive fixed point $u_\ast(L_{\min})$, and it is
attracting from any $v_0\in(0,1)$: on $(0,u_\ast)$ we have $F_{L_{\min}}(v)>v$ and on $(u_\ast,1)$ we have
$F_{L_{\min}}(v)<v$, so $v_t$ is monotone and bounded, hence convergent, and its limit is a fixed point,
which can only be $u_\ast$. Therefore $\liminf_t u_t\ge\lim_t v_t=u_\ast(L_{\min})$.
\end{proof}

\subsection*{Proof of Theorem~\ref{thm:washout} (subcritical coupling washes out the mixture)}

\begin{proof}

We have
\[
u_{t+1}
=
\tanh\!\left(\frac{L_tu_t + b_t}{2}\right).
\]

Using $|\tanh(x)|\le |x|$ and $0\le L_t\le L_{\max}<2$,
\[
|u_{t+1}|
\le
\frac{L_t}{2}|u_t| + \frac{|b_t|}{2}
\le
\frac{L_{\max}}{2}|u_t| + \frac{b_{\max}}{2}.
\]

Iterating this inequality gives
\[
|u_T|
\le
\left(\frac{L_{\max}}{2}\right)^T |u_0|
+
\frac12
\sum_{t=0}^{T-1}
\left(\frac{L_{\max}}{2}\right)^{T-1-t}|b_t|
\]
When $b_t = 0$, this reduces to
\[
|u_T|
\le
\left(\frac{L_{\max}}{2}\right)^T |u_0|
\]

Since
\[
\frac{L_{\max}}{2}<1,
\]
the right-hand side converges to zero as $T\to\infty$. Therefore
\[
u_T\to0
\]
independently of the initial mixture.

\end{proof}

\begin{remark}[Persistent nonzero field]

The theorem shows that the contribution of the initial mixture decays
geometrically under subcritical coupling. A persistent context-dependent
field, however, can sustain a nonzero residual state.

Indeed, if $|b_t|\le b_{\max}$ for all $t$, then, writing
\[
\rho=\frac{L_{\max}}{2}<1,
\]
we have
\[
|u_T|
\le
\rho^T|u_0|
+
\frac{b_{\max}}{2}
\sum_{t=0}^{T-1}\rho^{T-1-t}.
\]

Since
\[
\sum_{t=0}^{T-1}\rho^{T-1-t}
=
\frac{1-\rho^T}{1-\rho},
\]
we obtain
\[
|u_T|
\le
\rho^T|u_0|
+
\frac{b_{\max}}{2}\frac{1-\rho^T}{1-\rho}.
\]

Taking $T\to\infty$ yields
\[
\limsup_{T\to\infty}|u_T|
\le
\frac{b_{\max}}{2(1-\rho)}
=
\frac{b_{\max}}{2-L_{\max}}.
\]

Hence a persistent field can prevent convergence to the balanced state,
but under subcritical coupling its influence remains bounded by the
strength of the field relative to the distance from the critical
threshold.

\end{remark}

\subsection*{Proof of Theorem~\ref{thm:drift} (context-dimension lower bound)}

\begin{proof}
Write $\widetilde a(x)=f\!\left(\bar a,r(x)\right)$ for the corrected mixture at
contextual representation $x$, and define
\[
    F(x)
    =
    \Phi\!\left(x,\widetilde a(x)\right).
\]
Eq.~\ref{eq:anchor} says exactly that $F(x)=\bar a$ for every $x$ in a
neighborhood of $x_0$. A map that is constant on an open set has vanishing
derivative there, so $DF(x_0)=0$.

Throughout, $D_x\Phi$ and $D_a\Phi$ denote the partial derivatives of $\Phi$ with
respect to its first and second arguments, matching the notation of
Eq.~\ref{eq:context-rank}, and every derivative below is evaluated at the anchor
point $\left(x_0,\widetilde a(x_0)\right)=\left(x_0,f(\bar a,r(x_0))\right)$. Note
that $x$ enters $F$ through both arguments of $\Phi$, so the two contributions must
be summed. Since $\Phi$, $f$ and $r$ are continuously differentiable, the chain
rule gives
\[
    0
    =
    DF(x_0)
    =
    D_x\Phi
    +
    D_a\Phi\,
    D_rf\,
    D_xr ,
\]
and therefore
\[
    D_x\Phi
    =
    -\,D_a\Phi\,D_rf\,D_xr .
\]
The right-hand side factors through $D_xr$, and $r$ takes values in
$\mathbb{R}^m$, so $\operatorname{rank}(D_xr)\le m$. A product's rank is at most
the rank of any factor, hence
\[
    \operatorname{rank}\!\left(D_x\Phi\right)
    \le
    \operatorname{rank}\!\left(D_xr\right)
    \le m ,
\]
which is Eq.~\ref{eq:context-rank}. If in addition local variations of the
contextual representation move the output independently in all $K-1$ mixture
directions --- that is, if $D_x\Phi$ has full rank $K-1$ --- then $m\ge K-1$.
\end{proof}